\documentclass[final]{alt2023} 

\title[On Computable Online Learning]{On Computable Online Learning}
\usepackage{times}
\usepackage{enumitem}
\usepackage{mathtools}
\usepackage{bm}

\usepackage[hyphenbreaks]{breakurl}
\hypersetup{breaklinks=true}

\DeclarePairedDelimiter{\rlangle}{\langle}{\rangle}

\DeclareMathOperator*{\Exp}{\mathbb{E}}
\DeclareMathOperator*{\Prb}{\mathbb{P}}

\newcommand{\HH}[0]{\mathcal{H}}
\newcommand{\X}[0]{\mathcal{X}}
\newcommand{\Y}[0]{\mathcal{Y}}
\newcommand{\SSS}[0]{\mathbb{S}}
\newcommand{\I}[0]{\mathcal{I}}
\newcommand{\T}[0]{\mathcal{T}}
\newcommand{\D}[0]{\mathcal{D}}

\newcommand{\A}[0]{\Y ^ {\SSS \times \X}}

\newcommand{\N}[0]{\mathbb{N}}
\newcommand{\Q}[0]{\mathbb{Q}}

\newcommand{\Yb}[0]{\{0,1\}}

\newcommand{\Ldim}[0]{\text{Ldim}}
\newcommand{\dom}[0]{\text{dom}}
\newcommand{\rng}[0]{\text{rng}}

\newcommand{\one}{\text{\usefont{U}{bbold}{m}{n}1}}
\MakeRobust{\one}

\newcommand{\concat}[2]{{#1}^\frown\!{#2}}
\newcommand{\charfn}[1]{\chi_{\raisebox{-.5ex}{$\scriptstyle{#1}$}}}

\altauthor{%
 \Name{Niki Hasrati} \Email{nhasrati@uwaterloo.ca}\\
 \addr Cheriton School of Computer Science, University of Waterloo
 \AND
 \Name{Shai Ben-David} \Email{shai@uwaterloo.ca}\\
 \addr Cheriton School of Computer Science, University of Waterloo and Vector Institute%
}

\begin{document}

\maketitle

\begin{abstract}%
    We initiate a study of computable online (c-online) learning, which we analyze under varying requirements for ``optimality'' in terms of the mistake bound. Our main contribution is to give a necessary and sufficient condition for optimal c-online learning and show that the Littlestone dimension no longer characterizes the optimal mistake bound of c-online learning.
    Furthermore, we introduce anytime optimal (a-optimal) online learning, a more natural conceptualization of ``optimality'' and a generalization of Littlestone's Standard Optimal Algorithm. We show the existence of a computational separation between a-optimal and optimal online learning, proving that a-optimal online learning is computationally more difficult.
    Finally, we consider online learning with no requirements for optimality, and show, under a weaker notion of computability, that the finiteness of the Littlestone dimension no longer characterizes whether a class is c-online learnable with finite mistake bound. A potential avenue for strengthening this result is suggested by exploring the relationship between c-online and CPAC learning, where we show that c-online learning is as difficult as improper CPAC learning.
\end{abstract}


\section{Introduction}

Motivated by recent work on computable PAC (CPAC) learning \citep{agarwal-2020, agarwal-2021, sterkenburg2022characterizations}, we initiate a study of computable online (c-online) learning, where learners and their output hypotheses are required to be computable. 
As stated in Littlestone's seminal paper \citeyearpar[p. 289]{littlestone1988learning}, the original definition of online learning was limited to finite domains and hypothesis classes to avoid ``computability issues.'' Although Littlestone's results are easily extendable to the infinite setting \citep[see][Chapter 21]{shalev-schwartz-and-ben-david-2014}, an implicit assumption is that learners are functions, not necessarily computable, that map input samples to output hypotheses. Indeed, this assumption is implicit in many recent advances in online learning---for example, the equivalence between online learning and differentially private PAC learning \citep{alon2020private} and the characterizations of proper online learning \citep{chase2020bounds, hanneke2021online} and agnostic online learning \citep{ben-david-2009-agnostic}.  A further motivation for the study of computable learning stems from recent work on the undecidability of learning, where the authors state that ``the source of the problem is in defining learnability as the existence of a learning function rather than the existence of a learning algorithm'' \cite[p. 48]{ben2019learnability}.

A key result in online learning is that the Littlestone dimension characterizes the mistake bound of optimal online learners \citep[Theorem 3]{littlestone1988learning}. It is therefore natural to ask whether this fundamental result still holds in the computable setting. In this work, we formalize and investigate computable online learning under different notions of ``optimality'' in terms of the mistake bound.

Our main contribution is to give a necessary and sufficient condition for optimal c-online learning (Section \ref{subsection: characterizing optimal c-online learning}), the proof of which relies on expanding the concept of significant points introduced by \citet{frances1998optimal}. Using this condition, we show that the Littlestone dimension no longer characterizes the optimal mistake bound of c-online learning (Section \ref{subsection: impossibility result for optimal c-online learning}). In particular, we construct a class with finite Littlestone dimension for which no optimal online learner is computable. We also provide a positive result for the learnability of Littlestone dimension 1 classes in the computable setting (Section \ref{subsection: characterizing optimal c-online learning}).

 Additionally, we introduce a notion of anytime optimal (a-optimal) online learning which captures the optimality property displayed by Littlestone's Standard Optimal Algorithm (Sections \ref{section: anytime optimal online learning}, \ref{subsection: properties of anytime optimal online learners}). Although optimal and a-optimal online learning are equivalent in the standard online learning model, we prove a computational separation between the two, showing that a-optimal online learning is computationally more difficult than optimal online learning. Specifically, we construct a class that is optimally but not a-optimally c-online learnable (Section \ref{subsection: computational gap between optimal and a-optimal c-online learning}). 

A corollary of Theorem 3 from \citet{littlestone1988learning} is that the finiteness of the Littlestone dimension characterizes whether a class is online learnable at all---that is, whether it is online learnable with finite mistake bound. However, we show the existence of a ``weakly computable'' class with finite Littlestone dimension for which no computable online learner achieves finite mistake bound (Section \ref{subsection: impossibility result for c-online learning}). A potential avenue for strengthening this result is suggested in Section \ref{subsection: connection between c-online and CPAC learning}, where we explore the relationship between c-online and improper CPAC learning.

The paper is structured as follows. Section \ref{section: general background} provides the general background and notation needed from online learning and computability theory. Section \ref{section: setup and definition} introduces our main definitions of a-optimal online learning, optimally significant inputs, and c-online learning. The last three sections analyze c-online learning under increasingly looser notions of ``optimality''---Section \ref{section: anytime optimal c-online learning} considers a-optimal c-online learning, Section \ref{section: optimal c-online learning} optimal c-online learning, and Section \ref{section: c-online learning} c-online learning. 

\section{General Background}
\label{section: general background}

This section provides the required background from online learning 
(Section \ref{subsection: online learning})
and computability theory (Section \ref{subsection: computability}).

\subsection{Online Learning}
\label{subsection: online learning}

We first give an informal description of the online learning model and then introduce the formal notation that will be used throughout the paper. The definitions in this section are based on those given in \citet[Chapter 21]{shalev-schwartz-and-ben-david-2014}.

Introduced in \citet{littlestone1988learning}'s seminal work, online learning takes place in rounds. Informally, at each round $t$, an adversary presents the learner with some point $x_t$, the learner makes a prediction $p_t$, and the adversary reveals the true label $y_t$. The goal of the learner is to minimize the number of mistakes it makes. Clearly, with no further restrictions, the adversary could contradict the learner at each time step and cause an unbounded number of mistakes. It is therefore assumed that the learner has access to a class of hypotheses and that the sequence of examples presented by the adversary is consistent with some hypothesis from this class.
    
Formally, let $\X$ be the \textit{domain set} and $\Y=\Yb$ be the \textit{label set}. A \textit{hypothesis} is a function $h:\X\to\Y$ and a \textit{hypothesis class} is a set of hypotheses $\HH \subseteq \Y^\X$. 
The \textit{support} of a hypothesis $h$ is $h^{-1}(1) = \{x: h(x) = 1\}$. Given a set $E \subseteq \X$, the \textit{characteristic function} of $E$ is $\charfn{E}: x \mapsto \one_{[x \in E]}$. A \textit{sample} $S\in\SSS = \cup_{T\in\N}(\X\times\Y)^T$ is a finite sequence of labeled domain instances, where the \textit{empty sample} is denoted by $\varepsilon$. Given a sample $S=((x_i,y_i))_{i=1}^{T}$, let $S_n = ((x_i,y_i))_{i=1}^{n}$ be the \textit{length-$n$ prefix} of $S$, where $0 \leq n \leq T$. Denote by $\concat{S}{S^\prime}$ the concatenation of two samples $S, S^\prime \in \SSS$. The \textit{empirical loss of a hypothesis} $h$ with respect to a sample $S$ is defined as $L_S(h)=\sum_{t=1}^T \one_{[h(x_t) \neq y_t]}$.  The \textit{empirical loss of a hypothesis class} $\HH$ is $L_S(\HH) = \inf_{h\in\HH}L_S(h)$. The set of all samples that are \textit{$\HH$-realizable} is denoted by $\SSS_\HH = \{S \in \SSS: L_S(\HH)=0\}$. Given a sample $S$, define $\HH_S = \{h \in \HH: L_S(h) = 0\}$ as the set of all hypotheses from $\HH$ that are \textit{consistent} with $S$. For some labeled instance $(x,y)\in\X\times\Y$, let $\HH^{(x,y)} = \{h \in \HH: h(x) = y\}$. Furthermore, define $[n]=\{x\in\N: 1 \leq x \leq n\}$, where $n\in\N$.

\begin{definition}[online learner]
    An \textup{online learner} is a function $A \in \Y ^ {\SSS \times \X}$ that takes an \textup{input history} $S \in \SSS$ and a domain instance $x\in\X$ as input and predicts $A(S,x) \in \Yb$. Given a sample $S=((x_t,y_t))_{t=1}^T$, representing one run of the online learning process, at time step $t\in[T]$, $A$'s \textup{prediction} is $A(S_{t-1},x_t)$, its \textup{output hypothesis} is $A(S_{t-1},\cdot) \in \Y^\X$, and its \textup{version space} is $\HH_{S_{t-1}}$.
\end{definition}

\begin{definition}[mistake bound]
    The number of mistakes made by an online learner $A$ on a sample $S=((x_1,y_1),\ldots,(x_T,y_T))$ is $M_A(S) = \sum_{t=1}^T \one_{[A(S_{t-1}, x_t) \neq y_t]}$. The \textup{mistake bound} of $A$ with respect to a hypothesis class $\HH$ is $M_A(\HH) = \sup_{S \in \SSS_\HH} M_A(S)$---that is, the most that $A$ errs on any $\HH$-realizable sample. The \textup{optimal mistake bound} of $\HH$ is $M(\HH) = \inf_{A \in \A} M_A(\HH)$.
\end{definition}

\begin{definition}[online learnable class]
    A hypothesis class $\HH$ is \textup{online learnable} if $M(\HH) < \infty$.
\end{definition}

\begin{definition}[optimal online learner]
    An online learner $A$ is an \textup{optimal online learner} for a hypothesis class $\HH$ if $M_A(\HH) = M(\HH)$.
\end{definition}

\begin{definition}[$\HH$-shattered tree]
    Let $\HH \subseteq \Yb^\X$ and $d\in\N$. We say that $(x_1,\ldots,x_{2^d-1}) \in \X^{2^d-1}$ is an $\HH$\textup{-shattered tree} of depth $d$ if, for every $(y_1,\ldots,y_d)\in\Yb^d$, there exists $h\in\HH$ such that for all $j \in [d]$ we have that $h(x_{i_j}) = y_j$, where $i_j = 2^{j-1} + \sum_{k=1}^{j-1} y_k 2^{j-1-k}$. Let $\T^d_\HH$ denote the set of all $\HH$-shattered trees of depth $d$.
\end{definition}

\begin{remark}
    Intuitively, in the definition above, $(x_1,\ldots,x_{2^d-1}) \in \T^d_\HH$ represents a labeling of the nodes of a complete binary tree of depth $d$, with $x_i$ labeling $i^\text{th}$ node. Each $(y_1,\ldots,y_d)$ represents a different path through the tree starting from the root node $i_1 = 1$. If ${i_j}$ is the current node in the path, we go to the left child of $i_j$ if $y_j = 0$ and go to the right child if $y_j=1$.
\end{remark}

\begin{definition}[Littlestone dimension] 
    The \textup{Littlestone dimension} of a hypothesis class $\HH$ is the depth of the largest $\HH$-shattered tree. Formally, if $\HH\neq\emptyset$, $\Ldim(\HH) = \sup\{d\in\N: \T_\HH^d \neq \emptyset \}$ and $\Ldim(\emptyset)=-1$.
\end{definition}

\begin{remark}
    Note that, for any hypothesis class $\HH$, if $ \Ldim(\HH^{(x,r)}) = \Ldim(\HH)$ for some $x\in\X$ and $r\in\Yb$, we must have that $\Ldim(\HH^{(x,1-r)}) < \Ldim(\HH)$.
\end{remark}

\begin{definition}[Standard Optimal Learner]
     The \textup{Standard Optimal Learner} for a hypothesis class $\HH$ is defined as $SOL_\HH: (S,x) \mapsto \one_{\left[\Ldim\left(\HH_S^{(x,1)}\right) ~\geq~ \Ldim\left(\HH_S^{(x,0)}\right)\right]}$. 
\end{definition}

\begin{theorem}[{\citealp[Theorem 3]{littlestone1988learning}}]
\label{theorem: optimal mistake bound of online learning}
    Given any hypothesis class $\HH$, $M(\HH) = \Ldim(\HH)$. In particular, for every online learner $A$, $M_A(\HH) \geq \Ldim(\HH)$ and $M_{SOL_\HH}(\HH) = \Ldim(\HH).$\footnote{Although {\citet[Theorem 3]{littlestone1988learning}} only considers finite classes, the result is easily extendable to infinite classes if the learners are not required to be computable \citep[see][Corollary 21.8]{shalev-schwartz-and-ben-david-2014}}
\end{theorem}

\subsection{Computability}
\label{subsection: computability}

We use notation given by \citet{soare2016turing}. Let $\{P_e\}_{e\in\N}$ and $\{\varphi_e\}_{e\in\N}$ be effective numberings of all Turing machines and all \textit{partial computable (p.c.) functions}, respectively. If $P_e$ halts on input $x$ and outputs $y$, we write  $\varphi_e(x)=y$ and say that $\varphi_e(x)$ \textit{converges} (denoted $\varphi_e(x)\downarrow$). Otherwise, $\varphi_e(x)$ \textit{diverges} (denoted $\varphi_e(x)\uparrow$). The domain of $\varphi_e$ is $\dom(\varphi_e) = \{x: \varphi_e(x)\downarrow\}$ and its range is $\rng(\varphi_e) = \{\varphi_e(x): \varphi_e(x)\downarrow\}$. If $\dom(\varphi_e) = \N$, $\varphi_e$ is a \textit{total computable (t.c.) function} (abbreviated \textit{computable function}). We also extend this notation to $n$-place p.c. functions, where $\varphi_e^{(n)}$ is the p.c. function of $n$ variables computed by $P_e$ and $\varphi_e$ denotes $\varphi_e^{(1)}$. A set $E$ is \textit{recursively enumerable (r.e.)} if it can be effectively enumerated---that is, if it is the domain of some p.c. function. 
$E$ is \textit{decidable} if its characteristic function, $\charfn{E}: x \mapsto \one_{[x \in E]}$, is computable. 
The \textit{restriction} of $\varphi_e$ to an r.e. set $X$ 
is the p.c. function $\varphi_e|_X$, where $\varphi_e|_X(x)$ equals $\varphi_e(x)$ if $x \in X \cap \dom(\varphi_e)$ and is undefined otherwise. We say $\varphi_{e_2}$ is a \textit{p.c. extension} of $\varphi_{e_1}$ if $\varphi_{e_2}|_{\dom(\varphi_{e_1})} = \varphi_{e_1}$. 

The \textit{canonical index} of a finite set $F\subset\N$ is an integer $y$ that explicitly specifies all elements of $F$, and $D_y$ denotes the finite set with canonical index $y$.\footnote{Specifically, the \textit{canonical index} of a finite set $F\subset\N$ is the integer $y = \sum_{x \in F} 2^x$. The elements of the finite set with canonical index $y$, $D_y$, are the positions of the ``on'' bits in $y$'s binary expansion.} Furthermore, given a sequence $Z \in \cup_{n\in\N}\N^n$, we let $\rlangle{Z}$ denote the encoding of $Z$ by a standard 1:1 computable function from $\cup_{n\in\N}\N^n$ to $\N$. In a slight abuse of notation, we extend this notation to apply when $Z \in \SSS$.\footnote{To be explicit, given an $n$-tuple of integers $Z=(z_1,\ldots,z_n)$, we have that $\rlangle{Z}= \Pi_{i=1}^n p_i^{z_i+1}$, where $p_i$ is the $i$th prime number. Similarly, given a sample $S=((x_1,y_1),\ldots,(x_n,y_n))$, we define $\rlangle{S}=\rlangle{(x_1,y_1,\ldots,x_n,y_n)}$. Note that any 1:1 partially computable function is computably invertible on its range, so $Z$ and $S$ are computably recoverable given $\rlangle{Z}$ and $\rlangle{S}$ respectively.}
Additionally, for a set $X$ of such integer sequences, we define $\rlangle{X} = \{\rlangle{Z}: Z \in X\}$.

\section{Setup and definitions}
\label{section: setup and definition}

This section introduces our main definitions of anytime optimal online learning (Section \ref{section: anytime optimal online learning}), optimally significant inputs (Section \ref{subsection: (anytime) optimally significant inputs}), and c-online learning (Section \ref{subsection: computable online learning definition}).

\subsection{Anytime optimal online learning}
\label{section: anytime optimal online learning}
We present a notion of anytime optimal online learning, which we claim is a more natural conceptualization of ``optimality'' when referring to online learning. 

As a motivating example, consider the class $\HH_d = \{\charfn{[n]}\}_{n=1}^{2^d}$ over the domain $\X = \N$, where  $2 < d < \infty$ (recall that $[n]=\{1,2,\ldots,n\}$ and $\charfn{A}$ is the characteristic function of the set $A \subseteq \N$). Further define $E = \{2^d+i\}_{i=1}^{d-1}$ and let $\HH_d^\prime = \HH_d  \cup \{\charfn{E}\}$. That is, $\HH_d$ is a set of $2^d$ thresholds over the natural numbers and $E$ is a set of $d-1$ distinct domain instances that are not given the label 1 by any $h\in\HH_d$. It is easy to verify that $\Ldim(\HH_d^\prime) = \Ldim(\HH_d) = d$. Now, let $A$ be the learner that behaves as follows: for all inputs $(S,x) \in \SSS \times \X$, $A(S,x) = SOL_{\HH^\prime_d}(S,x)$ if $S$ is $\HH_d$-realizable and $A(S,x) = 0$ otherwise. Note that $A$ is still an optimal online learner for $\HH_d^\prime$ as it errs no more than $d$ times on any $\HH_d^\prime$-realizable sample; however, on the $\HH_d^\prime$-realizable sample $((x,1))_{x \in E}$, $A$ errs $d-1$ times while $SOL_{\HH_d^\prime}$ only errs at time step 1. It is clear that any $\HH_d^\prime$-realizable sample that contains some $x \in E$ with the label 1 can only be realized by $\charfn{E}$; hence, a ``truly optimal'' learner should incur no mistakes after seeing any $x \in E$ with the label 1.

The above example illustrates a gap between the commonly accepted definition of optimal online learning and the stricter optimality displayed by the Standard Optimal Learner. We define our notion of anytime optimal online learning below, where the learner makes the optimal number of mistakes even after conditioning on a given input sample. The properties of anytime optimal online learning are further explored in Section \ref{subsection: properties of anytime optimal online learners}.

\begin{definition}[post-$S$ mistake bound]
    Given a hypothesis class $\HH$, an online learner $A$, and an $\HH$-realizable sample $S\in\SSS_\HH$, we define the \textup{post-$S$ mistake bound} of $A$ with respect to $\HH$ as
    $$
    M_A^S(\HH) = \sup_{\substack{S^\prime ~ \in ~ \SSS: \\ \concat{S}{S^\prime} ~ \in ~ \SSS_\HH}} M_A(\concat{S}{S^\prime}) - M_A(S).
    $$
    That is, $M_A^S(\HH)$ is the most that $A$ can be made to err after witnessing $S$. The \textup{optimal post-$S$ mistake bound} of $\HH$ is defined as $M^S(\HH) = \inf_{A\in\A} M_A^S(\HH)$. 
\end{definition}

\begin{definition}[anytime optimal (a-optimal) online learner]
    An online learner $A$ is \textup{anytime optimal (a-optimal)} for a hypothesis class $\HH$ if $M_A^S(\HH) = M^S(\HH)$ for all 
    $S \in \SSS_\HH$. 
\end{definition}

\subsection{Significant inputs for optimal and a-optimal online learning}
\label{subsection: (anytime) optimally significant inputs}

\citet{frances1998optimal} introduced the concept of \textit{significant points}, points on which all optimal online learners agree on in the first time step of online learning. Formally, we say that $x \in \X$ is an \textit{optimally significant point} for online learning a class $\HH$ if $A(\varepsilon, x) = A^\prime(\varepsilon, x)$ for any two optimal online learners $A$ and $A^\prime$. Furthermore, Lemma 3 from \citet{frances1998optimal} characterizes all optimally significant points as follows: $x$ is an optimally significant point for $\HH$ iff there exists $r\in\Yb$ such that $\Ldim(\HH^{(x,r)}) = \Ldim(\HH)$. Moreover, $A(\varepsilon, x) = r$ for all online learners $A$ that are optimal w.r.t. $\HH$.
Below, we extend this definition to apply beyond the first time step.

\begin{definition}[optimally significant input]
    Let $\HH$ be any hypothesis class. We say that $(S,x) \in \SSS_\HH \times \X$ is an \textup{optimally significant input} for online learning $\HH$ if $A(S,x) = A^\prime(S,x)$ for any two optimal online learners  $A$ and $A^\prime$ for $\HH$. Let $\I_\HH$ be the set of all optimally significant inputs for $\HH$.
\end{definition}

\begin{definition}[anytime optimally (a-optimally) significant input]
    Let $\HH$ be any hypothesis class. We say that $(S,x) \in \SSS_\HH \times \X$ is an \textup{anytime optimally (a-optimally) significant input} for online learning $\HH$ if $A(S,x) = A^\prime(S,x)$ for any two a-optimal online learners  $A$ and $A^\prime$ for $\HH$.
\end{definition}

\subsection{Computable online learning}
\label{subsection: computable online learning definition}

When defining a computably online learnable hypothesis class, we require both the  class and the learner to conform to some notion of ``computability.''\footnote{The reader is referred to Section \ref{subsection: computability} for the relevant notation from computability theory.} 
Following the computable PAC (CPAC) setting \citep{agarwal-2020}, we let $\X = \N$, and assume, as a minimum, that the class consists of computable hypotheses. 
It is also desirable to assume an effective enumeration of (the encodings of) the hypotheses. 
A class $\HH \subset \Yb^\N$ of computable hypotheses is  \textit{recursively enumerably representable (RER)} if there exists an r.e. set $E \subset \N$ such that $\HH = \{\varphi_e: e \in E\}$. A class $\HH$ is \textit{decidably representable (DR)} if each $h \in \HH$ has finite support and $\{y: \exists h \in \HH ~ (D_y = h^{-1}(1))\}$ is a decidable set. 
Next, we define what it means for the learner itself to be computable. 

\begin{definition}[computable online (c-online) learner]
    Let $\HH \subset \Yb^\N$ be any class of computable hypotheses. A two-place p.c. function $A: \N^2 \to \N$ is a \textup{computable online (c-online) learner} for $\HH$, if, for every $\HH$-realizable sample $S \in \SSS_\HH$ and every domain instance $x\in\X$, $A(\rlangle{S}, x)\downarrow=y$ for some $y \in \Yb$. That is, $\dom(A) \supseteq \rlangle{\SSS_\HH} \times \X$ and $\rng(A|_{\rlangle{\SSS_\HH} \times \X}) \subseteq \Yb$.  
\end{definition}

\begin{definition}[computable optimal online learner]
    A \textup{computable optimal online learner} $A$ for a class $\HH \subset \Yb^\N$ of computable hypotheses is a c-online learner for $\HH$ with $M_A(\HH) = M(\HH)$.\footnote{Note that when $A$ is a c-online learner for a class $\HH$ of computable hypotheses, $M_A(S) = \sum_{t=1}^T \one_{[A(\rlangle{S_{t-1}},x_t) \neq y_t]}$ is well-defined for any $\HH$-realizable $S=((x_t,y_t))_{t=1}^T$. We can extend the notation for $M_A(\HH)$ and $M_A^S(\HH)$ similarly.}
\end{definition}

\begin{definition}[computable a-optimal online learner]
    A \textup{computable anytime optimal (a-optimal) online learner} $A$ for a class $\HH \subset \Yb^\N$ of computable hypotheses is a c-online learner for $\HH$ with $M_A^S(\HH) = M^S(\HH)$ for all $S \in \SSS_\HH$.
\end{definition}

\begin{definition}[computably online (c-online) learnable class]
    A class $\HH \subset \Yb^\N$ of computable hypotheses is \textup{computably online (c-online) learnable} if there exists a c-online learner $A$ for $\HH$ with $M_A(\HH) < \infty$.
\end{definition}

\begin{definition}[optimally c-online learnable class]
    A class $\HH \subset \Yb^\N$ of computable hypotheses is \textup{optimally c-online learnable} if there exists a computable optimal online learner for $\HH$.
\end{definition}

\begin{definition}[a-optimally c-online learnable class]
    A class $\HH \subset \Yb^\N$ of computable hypotheses is \textup{anytime optimally (a-optimally) c-online learnable} if there exists a computable a-optimal online learner for $\HH$.
\end{definition}

\section{Anytime optimal c-online learnability}
\label{section: anytime optimal c-online learning}

We start our analysis by considering the computability of a-optimal online learners. In Section \ref{subsection: computational gap between optimal and a-optimal c-online learning}, we show the existence of a computational separation between a-optimal and optimal online learning, proving that a-optimal online learning is computationally more difficult. Our proof relies on properties of a-optimal online learners presented in section \ref{subsection: properties of anytime optimal online learners} below.

\subsection{Properties of anytime optimal online learners}
\label{subsection: properties of anytime optimal online learners}

The following lemma gives a characterization of the optimal post-$S$ mistake bound of anytime optimal online learning in terms of the Littlestone dimension of the version space. The proof is implicit in the proof of Theorem 3 from \cite{littlestone1988learning}.

\begin{lemma}[characterizing the mistake bound of a-optimal online learning]
\label{lemma: characterizing the mistake bound of a-optimal online learning}
    Let $\HH$ be any hypothesis class. Then, for any $\HH$-realizable sample $S\in\SSS_\HH$, we have that $M^S(\HH) = \Ldim(\HH_S)$. In particular, for every online learner $A$, $M_A^S(\HH) \geq \Ldim(\HH_S)$ and $M_{SOL_\HH}^S(\HH) = \Ldim(\HH_S)$.
\end{lemma}

Informally, the next lemma states that an input is a-optimally significant iff it causes an ``imbalance'' in the Littlestone tree of the version space. 
Again, the proof is implicit in the proof of Theorem 3 from \citet[]{littlestone1988learning}.

\begin{lemma}[characterizing a-optimally significant inputs]
\label{lemma: characterizing a-optimally significant inputs}
    Let $\HH$ be any hypothesis class. Then, an input $(S,x) \in \SSS \times \X$ is a-optimally significant for $\HH$ iff $\Ldim(\HH_S^{(x,1)}) \neq \Ldim(\HH_S^{(x,0)})$. Furthermore, $A(S,x) = \arg\max_{r\in\Yb} \Ldim(\HH_S^{(x,r)})$ for all a-optimal online learners $A$ for $\HH$.
\end{lemma}

\subsection{Computational gap between optimal and a-optimal online learning}
\label{subsection: computational gap between optimal and a-optimal c-online learning}

In this section, we show that a-optimal online learning is computationally more difficult than optimal online learning.
In particular, we construct an RER class that is optimally but not a-optimally c-online learnable. This result is extended to the DR case in Appendix \ref{appendix: DR class not a-optimally c-online learnable}.

\begin{theorem}
\label{theorem: RER class not a-optimally c-online learnable}
    There exists an RER class $\HH \subset \Yb^\N$ of computable hypotheses with finite Littlestone dimension such that $\HH$ is optimally c-online learnable but not a-optimally c-online learnable.
\end{theorem}

\begin{proof}
    Consider the following class:
    $$
    \HH^{RER}_{halt} = \bigcup_{e\in\N} \left\{ \charfn{\{3e\}} \right\} \cup \bigcup_{\substack{e\in\N:\\ \varphi_e(e)\downarrow}} \left\{ \charfn{\{3e,~3e+1\}}, \charfn{\{3e,~3e+1,~3e+2\}} \right\}.
    $$
    For simplicity, let $\HH=\HH^{RER}_{halt}$. Note that $\HH$ is RER, each $h\in\HH$ is computable, and $\Ldim(\HH)<\infty$.
    
    Assume, by way of contradiction, that there exists a computable a-optimal online learner $A$ for $\HH$. For each $e\in\N$, let $S^e=((3e,1))$ and $x^e=3e+1$. Further define $f: e \mapsto A(\rlangle{S^e},x^e)$. First, note that $f$ is computable, since  for each $e\in\N$ the sample $S^e$ is $\HH$-realizable and $A(\rlangle{S^e},x^e)\downarrow$. Next, we show by Lemma \ref{lemma: characterizing a-optimally significant inputs} that each $(S^e,x^e)$ is an a-optimally significant input. Note that for any $e\in\N$,
    
    \noindent\begin{minipage}{.3\linewidth}
        $$\HH_{S^e}^{(x^e,0)} = \{\charfn{\{3e\}}\}$$
    \end{minipage}%
    \begin{minipage}{.1\linewidth}
        and
    \end{minipage}%
    \begin{minipage}{.6\linewidth}
        $$
        \HH_{S^e}^{(x^e,1)} = \begin{cases}
        	\{\charfn{\{3e,~3e+1\}}, \charfn{\{3e,~3e+1,~3e+2\}}\} &\text{ if $\varphi_e(e)\downarrow$}   \\
    	    \emptyset &\text{ otherwise. } \\
        \end{cases}
        $$
    \end{minipage}
    
    Therefore, if $\varphi_e(e)\downarrow$, $\Ldim(\HH_{S^e}^{(x^e,1)}) = 1 > \Ldim(\HH_{S^e}^{(x^e,0)}) = 0$ and $A(\rlangle{S^e},x^e)=1$. On the other hand, if $\varphi_e(e)\uparrow$, $\Ldim(\HH_{S^e}^{(x^e,0)}) = 0 > \Ldim(\HH_{S^e}^{(x^e,1)}) = -1$ and $A(\rlangle{S^e},x^e)=0$. Hence, $f$ is computable and equals $\charfn{\{e\in\N:~\varphi_e(e)\downarrow\}}$, contradicting the undecidability of the halting problem.
    
   Although $\HH$ is not a-optimally c-online learnable, we show the existence of a computable optimal online learner $B$ for $\HH$. It is easy to verify that $\Ldim(\HH) \geq 2$; hence, it suffices to show that $M_B(\HH)=2$. $B$ predicts 0 until, for some $e\in\N$, a mistake is made on $x_1 \in \{3e, ~3e+1,~3e+2\}$, at which point it matches $\charfn{\{3e,~3e+1,~x_1\}}$. If $x_1=3e+2$, $B$ will not err again. Otherwise, it could possibly err on $x_2 \in \{3e+1,~3e+2\}$. If $x_2=3e+2$, the target function must be $\charfn{\{3e,~3e+1,~3e+2\}}$; otherwise, if $x_2 = 3e+1$, the target function must be $\charfn{\{3e\}}$. In either case, $B$ errs no more than $\Ldim(\HH)=2$ times on any $\HH$-realizable sample.
\end{proof}

\section{Optimal c-online learnability}
\label{section: optimal c-online learning}

In this section, we loosen the requirement of a-optimality, turning our focus to all optimal online learners instead. We give a necessary and sufficient condition for when optimal c-online learning is possible (Section \ref{subsection: characterizing optimal c-online learning}) and show that  the Littlestone dimension no longer characterizes the mistake bound of optimal c-online learning (Section \ref{subsection: impossibility result for optimal c-online learning}). We also give a complete characterization of all optimally significant inputs (Section \ref{subsection: characterizing optimally significant inputs}), a result which is used in our main proofs.

\subsection{Characterizing optimally significant inputs}
\label{subsection: characterizing optimally significant inputs}
The following lemma gives a complete characterization of all optimally significant inputs. 

\begin{lemma}[characterizing optimally significant inputs]
\label{lemma: characterizing optimally significant inputs}
    Let $\HH$ be any hypothesis class satisfying $\Ldim(\HH) = d < \infty$. Let $S=((x_1,y_1),\ldots,(x_T,y_T))$ be any $\HH$-realizable sample and $x_{T+1}\in\X$ be any domain instance, where $T\in\N$. Then, $(S,x_{T+1})$ is a significant input w.r.t. optimal online learning $\HH$ iff the following conditions both hold:
        \begin{enumerate}
        \item for each $t\in[T+1]$,
        $\Ldim(\HH_{S_{t-1}})=\max_{r\in\Yb}\Ldim(\HH_{S_{t-1}}^{(x_t,r)})$, and
        \item for each $t\in[T]$, $\Ldim(\HH_{S_{t-1}}^{(x_t,y_t)}) \geq \Ldim(\HH_{S_{t-1}}) - 1$.
    \end{enumerate}
    Furthermore, $A(S_{t-1},x_t) = \arg\max_{r\in\Yb} \Ldim(\HH_{S_{t-1}}^{(x_t,r)})$ for all $t\in[T+1]$ and all optimal online learners $A$.
\end{lemma}

\begin{proof}
    It follows from Lemma \ref{lemma: equivalence of two conditions} (Appendix \ref{appendix: characterizing optimally significant inputs}) that conditions 1 and 2 above are equivalent to the following two conditions: 
    \begin{enumerate}
        \item[I.] $\Ldim(\HH_S) = \max_{r\in\Yb}\Ldim(\HH_S^{(x_{T+1},r)})$, and
        \item[II.] $M_A(S) = \Ldim(\HH) - \Ldim(\HH_S)$ for every online learner $A$ that is optimal for $\HH$.
    \end{enumerate}

    It remains to show that $(S, x_{T+1})$ is optimally significant iff conditions I and II hold. First, assume for the sake of contradiction that the two conditions hold but there exists an optimal online learner $A$ that predicts $A(S,x_{T+1})=1-r^*$, where $r^* = \arg\max_{r\in\Yb}\Ldim(\HH_S^{(x_{T+1},r)})$. Then, on the sample $S^* = \concat{S}{((x_{T+1},r^*))}$, $A$ makes $\Ldim(\HH)-\Ldim(\HH_S)+1$ mistakes and, by Lemma \ref{lemma: characterizing the mistake bound of a-optimal online learning}, can be made to err at least $
    \Ldim(\HH_S^{(x_{T+1},r^*)})=\Ldim(\HH_S)$ times after witnessing $S^*$, a contradiction. Furthermore, it follows from Lemma \ref{lemma: equivalence of two conditions} that $A(S_{t-1},x_t) = \arg\max_{r\in\Yb} \Ldim(\HH_{S_{t-1}}^{(x_t,r)})$ for all $t\in[T+1]$ and all optimal online learners $A$.

    Conversely, if $(S,x_{T+1})$ is optimally significant, there exists $r^* \in \Yb$ such that for all online learners $A$, if $A$ is optimal then $A(S,x_{T+1})=r^*$. For an a-optimal online learner $A^*$, let $A^\prime$ be the learner that agrees with $A^*$ on all inputs except $(S,x_{T+1})$. Since this single change in prediction causes $A^\prime$ to no longer be optimal, we must have that $M_{A^\prime}(S^*) + M_{A^\prime}^{S^*}(\HH) \geq d+1$, where $S^* = \concat{S}{((x_{T+1},r^*))}$. Note that $M_{A^\prime}(S^*) + M_{A^\prime}^{S^*}(\HH) = M_{A^*}(S)+1+M_{A^*}^{S^*}(\HH)$, so we must have that  $M_{A^*}(S)+M_{A^*}^{S^*}(\HH) \geq d$. However, since $A^*$ is a-optimal, the maximum values for $M_{A^*}(S)$ and $M_{A^*}^{S^*}(\HH)$ are $d - \Ldim(\HH_S)$ and $\Ldim(\HH_S)$ respectively. Hence, the inequality is only satisfied when both conditions I and II hold.
\end{proof}

\begin{corollary}[version space of optimally significant inputs]
\label{corollary: mistake bound on significant inputs}
    Let $\HH$ be a hypothesis class with $\Ldim(\HH) = d < \infty$ and let $(S,x) \in \I_\HH$ be any optimally significant input for $\HH$. Then, there exists $m\in\N$ such that $M_A(S) = m$ for all optimal online learners $A$ and $\Ldim(\HH_S) = d-m$. 
\end{corollary}

\subsection{Characterizing optimal c-online learning}
\label{subsection: characterizing optimal c-online learning}

In this section, we give a necessary and sufficient condition for optimal c-online learning in the RER setting (Corollary \ref{corollary: characterizing optimal c-online learnability}). The condition follows from Theorem \ref{theorem: computability of predictions on optimally significant inputs}, which shows that the predictions of all optimal online learners are computable on inputs that are optimally significant. Corollary \ref{corollary: inifinite RER Ldim 1 implies optimal c-online} shows that any infinite RER class of Littlestone dimension 1 is optimally c-online learnable.

\begin{theorem}[computability of optimally significant predictions]
\label{theorem: computability of predictions on optimally significant inputs}
    Let $\HH \subset \Yb^\N$ be any RER class of computable hypotheses with finite Littlestone dimension. Then, there exists a partial computable function $p_\HH^{sig}$ such that $p_\HH^{sig}(\rlangle{S},x) = A(S,x)$ for any optimally significant input $(S,x) \in \I_\HH$ and any optimal online learner $A$ for $\HH$.
\end{theorem}

\begin{proof}
    Let $\X=\N$ and $\HH \subset \Yb^\X$ be any RER class of computable hypotheses with $\Ldim(\HH) = d < \infty$. First, we show the existence of a Turing machine $M_\HH$ that behaves as follows: for any $S\in\SSS$, $x\in\X$, and $d^\prime\in\N$, if there exists $r\in\Yb$ for which $\Ldim(\HH_S^{(x,r)}) = d^\prime$ and $\Ldim(\HH_S^{(x,1-r)}) < d^\prime$, $M_\HH$ halts on input $(\rlangle{S},x,d^\prime)$ and outputs $r$. Note that for any RER class $\HH^\prime$ of computable hypotheses, the set $\rlangle{\T_{\HH^\prime}^{d^\prime}}$ of (the encodings of) all $\HH^\prime$-shattered trees of depth ${d^\prime}$ is r.e.. Therefore, since both $\HH_S^{(x,1)}$ and $\HH_S^{(x,0)}$ are RER, $M_\HH$ simultaneously runs the enumerators for $\T_{\HH_S^{(x,1)}}^{d^\prime}$ and $\T_{\HH_S^{(x,0)}}^{d^\prime}$ until one yields an output. If the enumerator for $\T^{d^\prime}_{\HH_S^{(x,y)}}$ yields an output first, $M_\HH$ halts and outputs $y$. Now, if there exists $r$ for which $\Ldim(\HH_S^{(x,r)}) = {d^\prime}$ and $\Ldim(\HH_S^{(x,1-r)}) < {d^\prime}$, we must have that $\T_{\HH_S^{(x,r)}}^{d^\prime} \neq \emptyset$ and $\T_{\HH_S^{(x,1-r)}}^{d^\prime} = \emptyset$; hence, $M_\HH$ will eventually halt and output $r$.

    Now, consider the Turing machine $P_\HH^{sig}$ that behaves as follows on any input $(\rlangle{S},x_{T+1})$, where $S = ((x_1,y_1),\ldots,(x_T,y_T))$ for some $T\in\N$: 1) initialize $m=0$; 2) for each $t\in[T+1]$, let $p_t$ be the result of running $M_{\HH}$ on input $(\rlangle{S_{t-1}}, x_t, d-m)$ and increment $m$ if $p_t \neq y_t$; 3) output $p_{T+1}$.
    
    We will show that if $(S,x_{T+1}) \in \I_\HH$, $p_t = \arg\max_{r\in\Yb} \Ldim(\HH_{S_{t-1}}^{(x_t,r)})$ for each $t \in [T+1]$; hence, by Lemma \ref{lemma: characterizing optimally significant inputs}, $p_\HH^{sig}$ is computed by $P_\HH^{sig}$. We proceed by induction on $t \in [T+1]$. If $t=1$, by lemma \ref{lemma: characterizing optimally significant inputs}, there exists $r_1\in\Yb$ such that $\Ldim(\HH_{S_0}^{(x_1,r_1)}) = d$ and $\Ldim(\HH_{S_0}^{(x_1,1-r_1)}) < d$. Therefore, $M_\HH$ halts on input $(\rlangle{S_0}, x_1, d)$ and outputs $r_1$. Now, consider any $\tau\in[T+1]$ such that the condition holds for all $t < \tau$. Then, $m_{\tau-1} = \sum_{t=1}^{\tau-1} \one_{[p_t \neq y_t]}$ is the number of mistakes that all optimal online learners make on $S_{\tau-1}$. Hence, by Corollary \ref{corollary: mistake bound on significant inputs}, $\Ldim(\HH_{S_{\tau-1}}) = d-m_{\tau-1}$ and, by Lemma \ref{lemma: characterizing optimally significant inputs}, there exists $r_\tau\in\Yb$ such that $\Ldim(\HH_{S_{\tau-1}}^{(x_\tau,r_\tau)}) = d-m_{\tau-1}$ and $\Ldim(\HH_{S_{\tau-1}}^{(x_\tau,1-r_\tau)}) < d-m_{\tau-1}$. Therefore, $M_\HH$ halts on $(\rlangle{S_{\tau-1}}, x_{\tau}, d-m_{\tau-1})$ and outputs $p_\tau = r_\tau$, as required.
\end{proof}

\begin{corollary}[characterizing optimal c-online learning]
\label{corollary: characterizing optimal c-online learnability}
    Let $\HH \subset \Yb^\N$ be any RER class of computable hypotheses with finite Littlestone dimension and let $p_\HH^{sig}$ be the partial computable function defined in Theorem \ref{theorem: computability of predictions on optimally significant inputs}. Then, $\HH$ is optimally c-online learnable iff there exists a p.c. extension $p_\HH^{real}$ of $p_\HH^{sig}$ such that $\dom(p_\HH^{real}) \supseteq \rlangle{\SSS_\HH} \times \X$ and $\rng(p_\HH^{real}|_{\rlangle{\SSS_\HH} \times \X}) \subseteq \Yb$.
\end{corollary}

\begin{corollary}[optimal c-online learnability of classes with Littlestone dimension 1]
\label{corollary: inifinite RER Ldim 1 implies optimal c-online}
    Let $\HH \subset \Yb^\N$ be any infinite RER class of computable hypotheses with $\Ldim(\HH) = 1$. Then, $\HH$ is optimally c-online learnable.
\end{corollary}
\begin{proof}
    By Corollary \ref{corollary: characterizing optimal c-online learnability}, it suffices to show that $\SSS_\HH \times \X \subseteq \I_\HH$. Let $T\in\N$ and consider any $\HH$-realizable sample $S=((x_1,y_1),\ldots,(x_T,y_T)) \in \SSS_\HH$ and any $x_{T+1} \in \X$. We will show that $(S,x_{T+1})$ satisfies Lemma \ref{lemma: characterizing optimally significant inputs} and is hence an optimally significant input for $\HH$. 
    Let $\tau \in [T]$ be the earliest time step such that $\Ldim(\HH_{S_{\tau-1}}) \neq \Ldim(\HH_{S_{\tau-1}}^{(x_\tau,y_\tau)})$. If no such time step exists, let $\tau = T+1$. Then, $\Ldim(\HH_{S_{t-1}}) = 1$ for all $t \leq \tau$ and, since $S$ is $\HH$-realizable, $\Ldim(\HH_{S_{t-1}})=0$ for all $\tau < t \leq T + 1$. Therefore, condition 2 of Lemma \ref{lemma: characterizing optimally significant inputs} is satisfied for all $t\in[T]$ and condition 1 is satisfied for all $t \neq \tau$. Now, since $\HH$ is infinite and at most one hypothesis is removed from the version space at each time step before $\tau$, $\HH_{S_{\tau-1}}$ is also infinite and there exists $r\in\Yb$ such that $\HH_{S_{\tau-1}}^{(x_\tau,r)}$ is infinite. Hence, $\Ldim(\HH_{S_{\tau-1}}) = \Ldim(\HH_{S_{\tau-1}}^{(x_\tau,r)}) = 1$ and condition 1 holds for $t=\tau$.
\end{proof}

\subsection{Littlestone dimension fails to characterize optimal mistake bound of online learning}
\label{subsection: impossibility result for optimal c-online learning}

In this section, we show that the Littlestone dimension no longer characterizes the mistake bound of optimal c-online learning. Specifically, we construct a DR class of computable hypotheses that has finite Littlestone dimension but is not optimally c-online learnable. Without the RER requirement, constructing such a class is not too difficult. In fact, the class $\HH_{halting} = \bigcup_{e\in\N: \varphi_e(e)\downarrow} \{\charfn{\{2e,~2e+1\}}\} \cup \bigcup_{e\in\N: \varphi_e(e)\uparrow} \{\charfn{\{2e\}}\}$, presented by \citet[Theorem 9]{agarwal-2020}, has Littlestone dimension 1 but any computable optimal online learner for this class would decide the halting problem.

\begin{theorem}
\label{theorem: DR class not optimally c-online learnable}
    There exists a DR class $\HH \subset \Yb^\N$ of computable hypotheses such that $\Ldim(\HH)<\infty$ but $\HH$ is not optimally c-online learnable. 
\end{theorem}

\begin{proof}
    For each $x\in\N$, let $\{C^{(x)}_i\}_{i\in\N: i > 0}$ be an effective enumeration of all halting computations starting from input $x$ \citep[see][Section 1.5.2]{soare2016turing}. Further define, for each $x\in\N$, the p.c. function $c_x$ such that if $P_e$ halts on input $x$, $C^{(x)}_{c_x(e)}$ is the halting certificate. That is, for each $e\in\N$,
	$$
		c_x(e)=\begin{cases}
			i &\text{if there exists $i$ s.t. $C^{(x)}_i$ is a halting computation for $P_e$ on input $x$}\\
			\text{undefined} &\text{otherwise.}
		\end{cases}
	$$

    Now, consider the following class:
    \begin{align*}
    	\HH_{ext}^{DR}
    	&= \bigcup_{\substack{e\in\N:~\varphi_e(0)\downarrow}} \left\{ \charfn{\left\{ 2^e,~ 2^e 3^{c_0(e)} \right\}} \right\} \\ 
    	&\cup \bigcup_{\substack{e\in\N:~\varphi_e(0)\downarrow} \text{ and } \varphi_e(e)\downarrow=1} \left\{ \charfn{\left\{ 2^e, ~2^e 5^{c_0(e)},~2^e 7^{c_e(e)} \right\}}, \charfn{\left\{ 2^e,~ 2^e 5^{c_0(e)},~ 2^e 11^{c_e(e)} \right\}} \right\} \\
    	&\cup \bigcup_{\substack{e\in\N:~\varphi_e(0)\downarrow} \text{ and } \varphi_e(e)\downarrow=0} \left\{ \charfn{\left\{ 2^e, ~2^e 5^{c_0(e)},~2^e 13^{c_e(e)} \right\}}, \charfn{\left\{ 2^e,~ 2^e 3^{c_0(e)},~ 2^e 13^{c_e(e)} \right\}} \right\}.
	\end{align*}
	
	For simplicity, let $\HH = \HH_{ext}^{DR}$. Note that each $h\in\HH$ is computable since $c_x(e)$ is evaluated only if $\varphi_e(x)\downarrow$. Furthermore, $\Ldim(\HH)=2$ (Appendix \ref{sub-appendix: Littlestone dimension of H^DR_ext}) and $\HH$ is DR 
    (Appendix \ref{sub-appendix: H^DR_ext is DR}).

    By Theorem \ref{theorem: computability of predictions on optimally significant inputs}, since $\HH$ is RER, there exists a p.c. function $p_\HH^{sig}$ such that $p_\HH^{sig}(\rlangle{S},x) = A(S,x)$ for any optimally significant input $(S,x) \in \I_\HH$ and any optimal online learner $A$ for $\HH$. For each $e\in\N$, let $S^e = ((2^e,1))$ and define the p.c. functions $x: e \mapsto 2^e 3^{c_0(e)}$ and $f: e \mapsto p_\HH^{sig}(\rlangle{S^e}, x(e))$. In Appendix \ref{sub-appendix: optimally significant inputs for H^DR_ext}, we show using Lemma \ref{lemma: characterizing optimally significant inputs} that $(S^e, x(e))$ is an optimally significant input for $\HH$ iff $\varphi_e(0)\downarrow$ and $\varphi_e(e)\downarrow\in\Yb$. Furthermore, for any $e\in\N$ such that $\varphi_e(0)\downarrow$ and $\varphi_e(e)\downarrow\in\Yb$, we have that
    $$
        f(e) = p_\HH^{sig}(\rlangle{S^e},x(e)) = \begin{cases}
            1 &\text{ if $\varphi_e(0)\downarrow$ and $\varphi_e(e)\downarrow=0$} \\
            0 &\text{ if $\varphi_e(0)\downarrow$ and $\varphi_e(e)\downarrow=1$.}
        \end{cases}
    $$
    
    Now, assume for the sake of contradiction that $\HH$ is optimally c-online learnable. Then, by Corollary \ref{corollary: characterizing optimal c-online learnability}, there exists a p.c. extension $p_\HH^{real}$ of $p_\HH^{sig}$ such that $\dom(p_\HH^{real}) \supseteq \rlangle{\SSS_\HH} \times \X$ and $\rng(p_\HH^{real}|_{\rlangle{\SSS_\HH} \times \X}) = \Yb$. It follows that the following function is also partial computable:
    $$
        g(e) = \begin{cases}
            0 &\text{if $e=0$} \\
            p_\HH^{real}(\rlangle{S^e},x(e)) &\text{ otherwise}.
        \end{cases}
    $$
    
    We will show that for any $e>0$ such that $\varphi_e(0)\downarrow$, we have that $g(e) \neq \varphi_e(e)$. First, if $\varphi_e(e)\downarrow\in\Yb$, $(S^e, x(e))$ is optimally significant for $\HH$ and $g(e)=f(e)=1-\varphi_e(e)$. Otherwise, if $\varphi_e(e)\uparrow$ or $\varphi_e(e)\downarrow\not\in\Yb$, we must have that $g(e)\downarrow\in\Yb$ since $S^e$ is $\HH$-realizable for any $e$ satisfying $\varphi_e(0)\downarrow$. Now, since $g$ is p.c. and each p.c. function has infinitely many indices, there exists $e^\prime>0$ such that $g = \varphi_{e^\prime}$. However, since $g(0)\downarrow$, this would imply the existence of some $e^\prime>0$ such that $\varphi_{e^\prime}(0)\downarrow$ and $g(e^\prime)=\varphi_{e^\prime}(e^\prime)$, a contradiction.
\end{proof}

\section{C-online learnability}
\label{section: c-online learning}

A corollary of Theorem \ref{theorem: optimal mistake bound of online learning} is that the finiteness of the Littlestone dimension characterizes whether a class is online learnable at all---that is, whether it is online learnable with finite mistake bound. Although the class $\HH^{DR}_{ext}$ presented in Theorem \ref{theorem: DR class not optimally c-online learnable} is not optimally c-online learnable, it is still c-online learnable by the learner that predicts 0 except on instances it has seen labeled 1. In this section, we analyze c-online learning when there is no requirement for optimality. As a first step, we construct a non-RER class of computable hypotheses that has finite Littlestone dimension but is not c-online learnable (Section \ref{subsection: impossibility result for c-online learning}). Next, we explore the connection between c-online and CPAC learning and suggest a potential avenue for strengthening the result to the RER setting (Section \ref{subsection: connection between c-online and CPAC learning}).

\subsection{Finite Littlestone dimension fails to characterize c-online learning}
\label{subsection: impossibility result for c-online learning}

The following theorem shows that, in the non-RER setting, the finiteness of the Littlestone dimension no longer characterizes c-online learnability.

\begin{theorem}
\label{theorem: non-RER class not c-online learnable}
    There exists a class $\HH \subset \Yb^\N$ of computable hypotheses such that $\Ldim(\HH)<\infty$ but $\HH$ is not c-online learnable.
\end{theorem}

\begin{proof}
    Recall that any c-online learner is a two-place partial computable function. The idea is to construct a class $\HH$ such that for any two-place p.c. function $A$ and for any input length $T$ there exists a hypothesis $h\in\HH$ and $T$ consecutive domain instances $x_1,\ldots,x_T\in\N$ such that, on the sample $S=((x_t,h(x_t)))_{t=1}^T$, we have that $A(\rlangle{S_{t-1}},x_t) \neq h(x_t)$ for all time steps $t \in [T]$. Hence, any c-online learner for $\HH$ will have an infinite mistake bound.
    
    Formally, define the functions $s_1 : n \mapsto \sum_{i=0}^n i$ and $s_2: n \mapsto \sum_{i=0}^n s_1(i)$. For each $i\in\N$ and $j \leq i$, let $N_i = \{n: s_2(i) \leq n < s_2(i+1)\}$ and $N_{i,j} = \{n: s_2(i) + s_1(j) \leq n < s_2(i) + s_1(j+1)\}$. Note that the natural numbers can be partitioned into disjoint sets $\N = \sqcup_{i\in\N} N_i$ and each $N_i$ can be further partitioned as $N_i = \sqcup_{j=0}^{i}N_{i,j}$. Let $I_1$, $I_2$, $I$, and $m$ be functions defined as follows: for each $i\in\N$, $j \leq i$, and $n \in N_{i,j}$, $I_1(n) = i$, $I_2(n)=j$, $I(n) = I_1(n) - I_2(n)$, and $m(n) = \min N_{i,j}$.
    
    Let $\{A_e\}_{e\in\N}$ be an effective numbering of all two-place p.c. functions and define the function $L: n \mapsto \one_{\left[A_{I(n)}\left(\rlangle{S^n},~n \right)\downarrow=0\right]}$, where $S^n = ((n^\prime, L(n^\prime)))_{n^\prime = m(n)}^{n-1}$. Now, let $\HH_{split} = \{h_i\}_{i\in\N}$, where
    $$
    h_i(n) = \begin{cases}
        L(n) &\text{ if } I_1(n) = i \\
        0 &\text{otherwise.}
    \end{cases}
    $$
    For simplicity, let $\HH = \HH_{split}$. Note that each $h_i$ is computable since $|h_i^{-1}(1)| \leq s_2(i) < \infty$. However, $\HH$ is not RER, since otherwise a Turing machine for computing $L$ would exist. Furthermore, $\Ldim(\HH)=1$ since each domain instance is given the label 1 by at most one $h\in\HH$.
    
    Now, assume for the sake of contradiction that $\HH$ is c-online learnable and let $A_e$ be a c-online learner for $\HH$. Since $A_e$ has finite mistake bound, there exists $M \in \N$ such that $M_{A_e}(\HH) \leq M$. However, we will show the existence of an $\HH$-realizable sample on which $A_e$ errs $M+1$ times.  Let $i = M+e$, $j = M$, and $S = ((n, h_i(n)))_{n = \min N_{i,j}}^{\max N_{i,j}}$. We will show that for each $t \in [|S|] = [M+1]$, we have that $A_e(\rlangle{S_{t-1}},n_t) = 1 - h_i(n_t)$, where $n_t = \min N_{i,j} + t - 1$ is the $t^\text{th}$ domain instance in $S$. By definition, since $n_t \in N_{i,j}$, we have that $I_1(n_t)=i$; hence, $h_i(n_t) = L(n_t) = \one_{\left[A_{I(n_t)}\left(\rlangle{S^{n_t}},~{n_t} \right)\downarrow=0\right]}$, where $S^{n_t} = ((n^\prime, L(n^\prime)))_{n^\prime = m(n_t)}^{n_t-1}$. Note that $I(n_t) = e$ and $S^{n_t} = S_{t-1}$. Therefore, $h_i(t) = \one_{[A_e(\rlangle{S_{t-1}}, n_t)\downarrow=0]}$. Now, since $A_e$ is a c-online learner for $\HH$ and $S_{t-1}$ is an $\HH$-realizable sample, we will always have that $A_e(\rlangle{S_{t-1}}, n_t)\downarrow\in\Yb$. Therefore, $h_i(n_t) = 1 - A_e(\rlangle{S_{t-1}}, n_t)$ for each $t \in [M+1]$ and $M_{A_e}(S) = M+1$, as required.
\end{proof}

\subsection{Connection between c-online and CPAC learning}
\label{subsection: connection between c-online and CPAC learning}

It is natural to ask whether Theorem \ref{theorem: non-RER class not c-online learnable} can be extended to the RER setting. That is, does there exist an RER class $\HH$ of computable hypotheses such that $\Ldim(\HH) < \infty$ but no c-online learner for $\HH$ achieves $M_A(\HH)<\infty$? In this section, we propose a potential avenue for addressing this question.

Recently, \citet{sterkenburg2022characterizations} proved a necessary condition for agnostic improper CPAC learnability and constructed an RER class of finite VC-dimension not satisfying this condition. In Lemma \ref{lemma: condition for non-c-online learnability in the agnostic setting}, we show that this condition is also necessary for agnostic c-online learnability. 
In particular, we show that any class that is agnostically c-online learnable is also agnostically improperly CPAC learnable but by a probabilistic learner (Lemma \ref{lemma: agnostic c-online implies improper CPAC}).

Thus far, we have been concerned with \textit{realizable c-online learners}---learners whose predictions are only guaranteed to be computable on realizable samples. We therefore extend the definition of agnostic online learning introduced by
\citet{ben-david-2009-agnostic}
to the computable setting. Let $\X = \N$ and $\HH \subset \Yb^\X$ be any class of computable hypotheses. An \textit{agnostic c-online learner} $A: \N^2 \to \Q \cap [0,1]$ is a two-place total computable function, where for any sample $S\in\SSS$ and any domain instance $x\in\X$, $A(\rlangle{S},x)$ is the probability of predicting the label 1 on the given input.\footnote{Since there exists a computable bijection between $\N$ and $\Q \cap [0,1]$, we can assume, without loss of generality, that $A$ is a valid computable function.}
The \textit{loss} of a hypothesis $h:\X\to[0,1]$ on a labeled instance $(x,y)$ is $\ell(h, (x,y)) = \Prb_{p \sim \text{Bernoulli}(h(x))} [p \neq y] = |h(x) - y|$. The \textit{expected regret} of an agnostic c-online learner $A$ with respect to $\HH$ and a sample size $T$ is
$
\Exp[R_A(\HH,T)] = \sup_{S=((x_t,y_t))_{t=1}^T} \left[\sum_{t=1}^T \ell(A_t, (x_t,y_t)) - \inf_{h\in\HH} \sum_{t=1}^T \ell(h, (x_t,y_t)) \right]$, where $A_t = A(\rlangle{S_{t-1}},\cdot)$.
The \textit{error} of $h:\X\to[0,1]$ w.r.t. a distribution $\D$ over $\X \times \Y$ is 
    $
    L_\D(h) = \Exp_{(x,y)\sim\D} \ell(h, (x,y))
    $
and the error of a hypothesis class $\HH$ w.r.t. $D$ is $L_\D(\HH)=\inf_{h\in\HH}L_\D(h).$

\begin{definition}[agnostic c-online learnable]
    A class $\HH \subset \Yb^\N$ of computable hypotheses is \textup{agnostically c-online learnable} if 
    there exists an agnostic c-online learner $A$
    whose expected regret grows sublinearly in the length of the input sample. That is, 
    $\lim_{T\to\infty} \frac{\Exp[R_A(\HH,T)]}{T}=0.$
\end{definition}

\begin{definition}[(agnostic) improper CPAC learnable by a probabilistic learner]
     A class $\HH$ of co\-mputable hypotheses is \textup{improperly CPAC learnable by a probabilistic learner (in the realizable setting)} if there exists a partial computable function $A:\N^2\to\Q\cap[0,1]$ and a function $m_\HH: (0,1)^2 \to \N$ such that $\dom(A) \supseteq \rlangle{\SSS_\HH} \times \X$ and for all $\epsilon,\delta\in(0,1)$, all $m \geq m_\HH(\epsilon,\delta)$, and all distributions $\D$ over $\X\times\Y$ that satisfy $L_\D(\HH)=0$, we have that with probability at least $1-\delta$ over $S\sim\D^m$, $L_\D(A_S) \leq L_\D(\HH) + \epsilon$, where $A_S = A(\rlangle{S},\cdot)$. We say that $\HH$ is \textup{\underline{agnostically} improperly CPAC learnable by a probabilistic learner} if $A$ is a total computable function and the above condition holds for any distributions $\D$ over $\X\times\Y$.
\end{definition}

\begin{lemma}[computable online-to-batch conversion]
\label{lemma: agnostic c-online implies improper CPAC}
    Let $\HH \subset \Yb^\N$ be any class of computable hypotheses that is (agnostically) c-online learnable. Then, $\HH$ is (agnostically) improperly CPAC learnable by a probabilistic learner.
\end{lemma}

\begin{proof}
    Let $A$ be an agnostic c-online learner for $\HH$. We use $A$ to construct an agnostic improper CPAC learner $B$ for $\HH$ that is probabilistic. For any $S=((x_t,y_t))_{t=1}^T$ and $x\in\X$, define
    $
    B(\rlangle{S},x) = \frac{1}{T} \sum_{t=1}^{T} A(\rlangle{S_{t-1}}, x).
    $
     We can think of $B$ as representing an algorithm that uniformly at random picks some $t\in[T]$ and outputs $A(\rlangle{S_{t-1}},\cdot)$ as its hypothesis. As required, $B$ is a computable function from $\N^2$ into $\Q \cap [0,1]$. The proof that $B$ is a PAC learner for $\HH$ 
     follows from the standard online-to-batch conversion argument \citep[see][Exercise 21.7.5]{kakade-and-tewari-2008,shalev-schwartz-and-ben-david-2014}.
     The proof can also be extended to the realizable setting.
\end{proof}

\begin{lemma}[necessary condition for agnostic c-online learnability]
\label{lemma: condition for non-c-online learnability in the agnostic setting}
    Let $\HH \subset \Yb^\N$ be any class of computable hypotheses that is agnostically c-online learnable. Then, $\HH$ satisfies the following two conditions: (1) $\Ldim(\HH) < \infty$ and (2) for sufficiently large $n$, there exists an algorithm $C_n$ that on any input $X \subset \X$ of size $n$, outputs a labeling $g: X \to \Yb$ for which $((x,g(x)))_{x \in X}$ is not $\HH$-realizable.
\end{lemma}
\begin{proof}
    The first condition follows from \citet{ben-david-2009-agnostic}, who showed that $\HH$ is agnostically online learnable in the standard setting iff $\Ldim(\HH)<\infty$. The second condition follows almost directly form \citet[Lemma 9]{sterkenburg2022characterizations}, who showed that if $\HH$ is agnostically improperly CPAC learnable, for sufficiently large $n$, there exists an algorithm $C_n$ satisfying the stated property. Their proof, which follows from the Computable No-Free-Lunch theorem \citep[Lemma 19]{agarwal-2020}, can also be extended to probabilistic learners.
    Hence, the result follows from Lemma \ref{lemma: agnostic c-online implies improper CPAC}.
\end{proof}

\noindent\textbf{Open Question} Is there an RER class of computable hypotheses with finite Littlestone dimension that is not c-online learnable? Lemma \ref{lemma: agnostic c-online implies improper CPAC} suggests one approach to addressing this question: constructing a class with finite Littlestone dimension that is not improperly CPAC learnable (by a probabilistic learner). Similarly, Lemma \ref{lemma: condition for non-c-online learnability in the agnostic setting} could be applied to construct a class that is not c-online learnable in the \textit{agnostic} setting.

In \appendixref{appendix: Ldim of H_init}, we show that the class $\HH_{init}$ presented by \citet[Theorem 10]{sterkenburg2022characterizations}---the only known RER class of finite VC-dimension that is not improperly CPAC learnable---has infinite Littlestone dimension. Hence, this class cannot be used to address the question stated above. It remains open whether there exists an RER class of computable functions that has finite Littlestone dimension but is not improperly CPAC learnable.

\section{Conclusion and Future Work}

In this paper, we investigate computable online learning under three different settings. First, we formalize anytime optimal (a-optimal) online learning, a natural conceptualization of ``optimality,'' and show that it is computationally more difficult than optimal online learning. Second, we give a necessary and sufficient condition for optimal c-online learning and prove that the Littlestone dimension no longer characterizes the optimal mistake bound of c-online learning. 
Finally, we demonstrate that, in the non-RER setting, the finiteness of the Littlestone dimension no longer determines whether a class is c-online learnable with finite mistake bound. Although this last result remains open in the RER setting, we show that it is equivalent to asking whether there exists an RER class of computable functions that has finite Littlestone dimension but is not improperly CPAC learnable.

As we have shown that some very fundamental results from online learning fail in the computable setting, it would be interesting for future work to explore computable online learning in various related settings---for example, agnostic online learning, proper online learning, and differentially private PAC learning.

Furthermore, similar to \citet{sterkenburg2022characterizations}'s characterization of proper CPAC learning, our characterization of optimal c-online learning relies on computability-theoretic concepts. A major remaining open problem is to find purely combinatorial characterizations of computable learnability.

\acks{
     We would like to thank CIFAR and the Vector Institute for their support: CIFAR for supporting Shai as a Canada AI CIFAR chair and the Vector Institute for supporting Niki through a research grant and Shai through a faculty appointment. We would also like to thank Alex Bie, Tosca Lechner, and Matt Regehr for interesting and helpful discussions.
}

\bibliography{alt2023}

\appendix

\section{Proof of Lemma \ref{lemma: characterizing optimally significant inputs}}
\label{appendix: characterizing optimally significant inputs}

\begin{lemma}
\label{lemma: equivalence of two conditions}
    Let $\HH$ be a hypothesis class such that $\Ldim(\HH) = d < \infty$. Let $S=((x_t,y_t))_{t=1}^T$ be any $\HH$-realizable sample and $x_{T+1}\in\X$ be any domain instance, where $T\in\N$. Then, the following conditions are equivalent:
        \begin{enumerate}
        \item[A.] For each $t\in[T]$, $\Ldim(\HH_{S_{t-1}})=\max\limits_{r\in\Yb}\Ldim(\HH_{S_{t-1}}^{(x_t,r)})$ and $\Ldim(\HH_{S_t}) \geq \Ldim(\HH_{S_{t-1}}) - 1$
        \item[B.] $M_A(S)=\Ldim(\HH)-\Ldim(\HH_S)$ for every online learner $A$ that is optimal for $\HH$.
    \end{enumerate}
    Furthermore, for all $t\in[T]$ and all optimal online learners $A$, we have that $A(S_{t-1},x_t) = \arg\max_{r\in\Yb} \Ldim(\HH_{S_{t-1}}^{(x_t,r)})$.
\end{lemma}
\begin{proof}
    (A $\implies$ B) Assume that condition A holds and let $A^*$ be an a-optimal online learner for $\HH$. Note that, by Lemma \ref{lemma: characterizing a-optimally significant inputs},  each $(S_{t-1},x_t)$ is an a-optimally significant input and $A^*(S_{t-1},x_t) = r^*_t = \arg\max_{r\in\Yb} \Ldim(\HH_{S_{t-1}}^{(x_t,r)})$. Hence, it follows from condition A that the Littlestone dimension of the version space decreases iff $A^*$ errs and decreases by at most one at each time step. Therefore, $M_{A^*}(S_t) = d - \Ldim(\HH_{S_t})$ for any $t\in[T]$.
    
    We will show that condition B holds by showing that, for each $t \in [T]$, every optimal online learner must agree with $A^*$ on $(S_{t-1},x_t)$. Assume for the sake of contradiction that there exists an optimal online learner $A$ such that for some $t \in [T]$, $A(S_{t-1},x_t)=1-r^*_t$. Let $\tau$ be the earliest such time step. Then, on the sample $\concat{S_{\tau-1}}{((x_\tau,r^*_\tau))}$, $A$ errs $M_{A^*}(S_{\tau-1}) + 1$ times. However, by Lemma \ref{lemma: characterizing the mistake bound of a-optimal online learning}, $A$ can be made to err at least $\Ldim(\HH_{S_{\tau-1}}^{(x_\tau, r^*_\tau)}) = \Ldim(\HH_{S_{\tau-1}}) = d - M_{A^*}(S_{\tau-1})$ more times, a contradiction.\\

    (B $\implies$ A) 
    Let $A^*$ be an a-optimal online learner such that $A^*(S_{t-1},x_t)=y_t$ for all $(S_{t-1},x_t)$ that are not a-optimally significant. That is, $A^*$ errs iff $\Ldim(\HH_{S_{t-1}}^{(x_t,y_t)}) < \Ldim(\HH_{S_{t-1}}^{(x_t,1-y_t)})$. Furthermore, $M_{A^*}(S) \leq d - \Ldim(\HH_S)$, as every time $A^*$ errs the Littlestone dimension of the version space decreases by at least one. We will show that if condition A does not hold, this inequality is strict.

    First, if there exists $t\in[T]$ such that $\Ldim(\HH_{S_{t-1}}^{(x_t,y_t)}) < \Ldim(\HH_{S_{t-1}})$ and $\Ldim(\HH_{S_{t-1}}^{(x_t,1-y_t)}) < \Ldim(\HH_{S_{t-1}})$, there are two cases. Either $A^*$ does not err at time step $t$ and the Littlestone dimension of the version space decreases by at least one, or $A^*$ errs and the Littlestone dimension of the version space decreases by at least two. Similarly, if there exists $t\in[T]$ such that $\Ldim(\HH_{S_t}) \leq \Ldim(\HH_{S_{t-1}}) - 2$, the Littlestone dimension of the version space goes down by at least one more than the number of mistakes made. In either case, $M_{A^*}(S) < d - \Ldim(\HH_S)$.
\end{proof}

\section{Proof of Theorem \ref{theorem: DR class not optimally c-online learnable}}

\label{appendix: DR class not optimally c-online learnable}

\subsection{Littlestone dimension of $\HH^{DR}_{ext}$}
\label{sub-appendix: Littlestone dimension of H^DR_ext}

\begin{lemma}
\label{lemma: Littlestone dimension of H^DR_ext}
    $\Ldim(\HH^{DR}_{ext}) = 2$.
\end{lemma}
\begin{proof}
    For simplicity, let $\HH = \HH^{DR}_{ext}$. First, we will show that $\Ldim(\HH) \geq 2$. Consider any three distinct indices $e_1, e_2, e_3 \in \N$ such that $\varphi_{e_i}(0)\downarrow$ for all $i \in [3]$ and $\varphi_{e_1}(e_1)\downarrow=1$. Then, the $\N$-labeled tree of depth 2 given by $2^{e_2} \leftarrow 2^{e_1} \rightarrow 2^{e_1} 3^{c_0(e_1)}$ is shattered by $\charfn{\{2^{e_3}, 2^{e_3}3^{c_0(e_3)}\}}$ , $\charfn{\{2^{e_2}, 2^{e_2}3^{c_0(e_2)}\}}$, $\charfn{\{2^{e_1}, 2^{e_1}5^{c_0(e_1)}, 2^{e_1}7^{c_{e_1}(e_1)}\}}$, $\charfn{\{2^{e_1}, 2^{e_1}3^{c_0(e_1)}\}} \in \HH$.
    
    Next, we will show that $\Ldim(\HH) \leq 2$ by showing the existence of a learner $B$ (not necessarily computable) which errs at most twice on any $\HH$-realizable sample. $B$ predicts 0 until (possibly) a mistake is made on $x_1$. There are two cases for $x_1$. If $x_1 = 2^ey^i$ for some $e,i\in\N$ s.t. $i>0$ and $y \in \{3, 5, 7, 11, 13\}$, $B$ matches $\charfn{\{2^e, 2^e y^i\}}$ until a mistake is potentially made on $x_2$, at which point it matches the target function $\charfn{\{2^e,2^e y^i,x_2\}}$ and does not err again. If $x_1 = 2^e$ for some $e\in\N$, there are three cases. If $\varphi_e(e)\downarrow=1$, $B$ matches $\charfn{\{2^e, 2^e 5^{c_0(e)}\}}$, if $\varphi_e(e)\downarrow=0$, $B$ matches $\charfn{\{2^e, 2^e 13^{c_e(e)}\}}$, and otherwise $B$ matches $\charfn{\{2^e, 2^e 3^{c_0(e)}\}}$. In either case, $B$ can be made to err at most once more. 
\end{proof}

\subsection{Proof that $\HH^{DR}_{ext}$ is DR}
\label{sub-appendix: H^DR_ext is DR}

\begin{lemma}
    $\HH^{DR}_{ext}$ is decidably representable.
\end{lemma}
\begin{proof}
    First, note that the set $\{(e, i, x): C_i^{(x)} \text{ is a halting certificate for $P_e$ on input $x$}\}$ is decidable by the following Turing machine $P_{cert}$. On any input $(e, i, x)$, after ensuring that $i>0$, $P_{cert}$ simulates running $P_e$ on input $x$ and checks each configuration that $P_e$ goes through against the corresponding one in $C_i^{(x)}$. If at any point the configurations are not the same or if there are no more configurations left to check from $C_i^{(x)}$, $P_{cert}$ halts and outputs $0$. Otherwise, if $P_e$ halts on input $x$ and all the configurations match, $P_{cert}$ halts and outputs 1. $P_{cert}$ is guaranteed to halt since $C_i^{(x)}$ is a finite sequence of configurations.

    Now, we will show that the set $\{y: \exists h \in \HH ~ (D_y = h^{-1}(1))\}$ is decidable by the following Turing machine $P$. Given the canonical index $y$ of any finite set as input, $P$ first decodes $y$ into its associated set $D_y$ and checks if $D_y$ equals any of the sets $\{2^e, 2^e 3^i\}$, $\{2^e, 2^e 5^i, 2^e 7^j\}$, $\{2^e, 2^e 5^i, 2^e 11^j\}$, $\{2^e, 2^e 5^i, 2^e 13^j\}$, $\{2^e, 2^e 3^i, 2^e 13^j\}$ for some $e, i, j \in \N$ such that $i, j > 0$. If not, $P$ halts and outputs $0$. Otherwise, if $D_y = \{2^e, 2^e 3^i\}$, $P$ halts and outputs the result of running $P_{cert}$ on $(e, i, 0)$. Otherwise, $P$ evaluates $P_{cert}$ on $(e, i, 0)$ and $(e, j, e)$ and, if either result is $0$, halts and outputs $0$. If both invocations of $P_{cert}$ yield 1, let $r$ be the result of evaluating $P_e$ on input $e$. $P$ outputs $1$ if $r=0$ and $2^e 13^j \in D_y$ or if $r=1$ and $2^e 13^j \not\in D_y$. Otherwise, it outputs $0$. 
\end{proof}

\subsection{Optimally significant inputs for $\HH^{DR}_{ext}$}
\label{sub-appendix: optimally significant inputs for H^DR_ext}

\begin{lemma}
    For each $e\in\N$, let $S^e = ((2^e,1))$ and define the p.c. function $x: e \mapsto 2^e 3^{c_0(e)}$. $(S^e, x(e))$ is a significant input w.r.t. optimal online learning $\HH^{DR}_{ext}$ iff $\varphi_e(0)\downarrow$ and $\varphi_e(e)\downarrow\in\Yb$. Furthermore, for any optimal online learner $A$ for $\HH^{DR}_{ext}$, if $\varphi_e(0)\downarrow$ and $\varphi_e(e)\downarrow=r$ for some $r\in\Yb$, $A(S^e,x(e)) = 1-r$
\end{lemma}

\begin{proof}
    Let $\HH = \HH^{DR}_{ext}$. First, consider any $e\in\N$ such that $\varphi_e(0)\downarrow$ and $\varphi_e(e)\downarrow\in\Yb$. We will show that $(S^e, x(e))$ is an optimally significant input by showing that it satisfies Lemma \ref{lemma: characterizing optimally significant inputs}. That is, we need to show that $\Ldim(\HH_{S^e}) \geq \Ldim(\HH) - 1$, $\Ldim(\HH) = \max_{r\in\Yb} \Ldim(\HH^{(2^e,r)})$, and $\Ldim(\HH_{S^e}) = \max_{r\in\Yb} \Ldim(\HH_{S^e}^{(x(e),r)})$. 
    
    By Lemma \ref{lemma: Littlestone dimension of H^DR_ext}, $\Ldim(\HH)=2$, and it is easy to verify that $\Ldim(\HH^{(2^e,0)}) = 2$ and $\Ldim(\HH_{S^e}) = 1$. Hence, the first two conditions are satisfied. For the third condition there are two cases. Note that for $r\in\Yb$, 
    $$
    \Ldim(\HH_{S^e}^{(x(e),r)}) = \begin{cases}
        \{\charfn{\{2^e, 2^e 5^{c_0(e)}, 2^e 7^{c_e(e)}\}}, \charfn{\{2^e, 2^e 5^{c_0(e)}, 2^e 11^{c_e(e)}\}}\} &\text{ if } r = 0 \text{ and } \varphi_e(e)\downarrow=1 \\
        \{\charfn{\{2^e, 2^e 3^{c_0(e)}\}}, \charfn{\{2^e, 2^e 3^{c_0(e)}, 2^e 13^{c_e(e)}\}}\} &\text{ if } r = 1 \text{ and } \varphi_e(e)\downarrow=0.
    \end{cases}
    $$
    
    Hence, $\Ldim(\HH_{S^e}^{(x(e),1-\varphi_e(e))}) = \Ldim(\HH_{S^e}) = 1$ and by Lemma \ref{lemma: characterizing optimally significant inputs}, $(S^e, x(e))$ is an optimally significant input and $A(S^e, x(e)) = 1 - \varphi_e(e)$ for any optimal online learner $A$, as required.
    
    Conversely, for any $e\in\N$ such that $\varphi_e(0)\uparrow$, $S^e$ is not $\HH$-realizable and $(S^e, x(e))$ cannot be an optimally significant input. Now, for any $e\in\N$ such that $\varphi_e(0)\downarrow$ but $\varphi_e(e)\not\in\Yb$, $\HH_{S^e} = \{\charfn{\{2^e, 2^e 3^{c_0(e)}\}}\}$ and $\Ldim(\HH_{S^e}) = 0 < \Ldim(\HH) - 1$. Hence, Lemma \ref{lemma: characterizing optimally significant inputs} is not satisfied and $(S^e, x(e))$ is not an optimally significant input.
\end{proof}

\section{Extending Theorem \ref{theorem: RER class not a-optimally c-online learnable} to the DR setting}
\label{appendix: DR class not a-optimally c-online learnable}

In this section, we extend Theorem \ref{theorem: RER class not a-optimally c-online learnable} to the DR setting. The technique is similar to that used in the proof of Theorem \ref{theorem: DR class not optimally c-online learnable}.

\begin{theorem}
    There exists a DR class $\HH \subset \Yb^\N$ of computable hypotheses with finite Littlestone dimension such that $\HH$ is optimally c-online learnable but not a-optimally c-online learnable. 
\end{theorem}

\begin{proof}
    For each $x\in\N$, let the p.c. function $c_x$ be defined as in Theorem \ref{theorem: DR class not optimally c-online learnable} and consider the following class:
	\begin{align*}
    	\HH_{halt}^{DR}
    	&= \bigcup_{\substack{e\in\N:~\varphi_e(0)\downarrow}} \left\{ \charfn{\left\{ 2^e,~ 2^e 3^{c_0(e)} \right\}} \right\} \\ 
    	&\cup \bigcup_{\substack{e\in\N:~\varphi_e(0)\downarrow} \text{ and } \varphi_e(e)\downarrow} \left\{ \charfn{\left\{ 2^e, ~2^e 5^{c_0(e)},~2^e 7^{c_e(e)} \right\}}, \charfn{\left\{ 2^e,~ 2^e 5^{c_0(e)},~ 2^e 11^{c_e(e)} \right\}} \right\}.
	\end{align*}
	
	For simplicity, let $\HH = \HH_{halt}^{DR}$. Since $|h^{-1}(1)| \leq 3$ for each $h\in\HH$, we have that $\Ldim(\HH) < \infty$. Furthermore, each $h\in\HH$ is computable since $c_x(e)$ is evaluated only if $\varphi_e(x)\downarrow$. 
	To show that $\HH$ is DR, the same proof technique presented in \appendixref{sub-appendix: H^DR_ext is DR} can be applied.

	Now, assume for the sake of contradiction that there exists a computable a-optimal online learner $A$ for $\HH$. For each $e\in\N$, let $S^e = ((2^e,1))$ and define the p.c. functions $x: e \mapsto 2^e 3^{c_0(e)}$ and $f: e \mapsto A(\rlangle{S^e}, x(e))$. We will show that
	
	$$
	f(e) = A(\rlangle{S^e}, x(e)) = \begin{cases}
		1 &\text{if } \varphi_e(0)\downarrow \text{ and } \varphi_e(e)\uparrow \\
		0 &\text{if } \varphi_e(0)\downarrow \text{ and } \varphi_e(e)\downarrow \\
		\text{undefined} &\text{if }  \varphi_e(0)\uparrow. 
	\end{cases}
	$$
	
	First, note that $f(e)\downarrow$ iff $\varphi_e(0)\downarrow$: if $\varphi_e(0)\downarrow$, $S^e$ is $\HH$-realizable and $c_0(e)\downarrow$; otherwise, $S^e$ is not $\HH$-realizable and $c_0(e)\uparrow$. Next, we show by Lemma \ref{lemma: characterizing a-optimally significant inputs} that if $\varphi_e(0)\downarrow$, we must have that $(S^e,x(e))$ is a-optimally significant for $\HH$. Note that for any $e$ such that $\varphi_e(0)\downarrow$ we must have that
	
	$$
        \HH_{S^e}^{(x(e),1)} = \left\{\charfn{\left\{2^e, 2^e 3^{c_0(e)}\right\}}\right\}
    $$
    
     and 
     
    $$
        \HH_{S^e}^{(x(e),0)} = \begin{cases}
            \emptyset &\text{ if } \varphi_e(e)\uparrow \\
            \left\{ \charfn{\left\{ 2^e, ~2^e 5^{c_0(e)},~2^e 7^{c_e(e)} \right\}}, \charfn{\left\{ 2^e,~ 2^e 5^{c_0(e)},~ 2^e 11^{c_e(e)} \right\}} \right\} &\text{ if } \varphi_e(e)\downarrow 
        \end{cases}
    $$
    
    Therefore, if $\varphi_e(0)\downarrow$ and $\varphi_e(e)\uparrow$, $\Ldim(\HH_{S^e}^{(x(e),1)}) = 0 > \Ldim(\HH_{S^e}^{(x(e),0)}) = -1$ and $f(e)=1$. On the other hand, if $\varphi_e(0)\downarrow$ and $\varphi_e(e)\downarrow$, $\Ldim(\HH_{S^e}^{(x(e),0)}) = 1 > \Ldim(\HH_{S^e}^{(x(e),1)}) = 0$ and $f(e)=0$. Next, we can use $f$ to construct the following p.c. function:
    
    $$
    g(e) = \begin{cases}
        1 &\text{ if } e = 0 \\
        1 &\text{ if } e > 0 \text{ and } f(e)=1 \\
        \text{undefined} &\text{ if } e > 0 \text{ and } f(e)=0 \text{ or } f(e)\uparrow.
    \end{cases}
    $$
    
    Since $g$ is a p.c. function, there exists $e$ such that $\varphi_e = g$. Furthermore, since each p.c. function has infinitely many indices, we can assume that $e>0$. Now, by definition of $g$, since $e>0$,
    
    $$
    g(e)\downarrow \iff f(e) = 1 \iff \varphi_e(0)\downarrow ~ \wedge ~ \varphi_e(e)\uparrow \iff g(e)\uparrow,
    $$
    
    contradicting the existence of an a-optimal c-online learner for $\HH$.
    
    Although $\HH$ is not a-optimally c-online learnable, we can show that there exists a computable optimal online learner $B$ for $\HH$. It is easy to verify that $\Ldim(\HH) \geq 2$; hence, it suffices to show that $M_B(\HH)=2=\Ldim(\HH)$. $B$ predicts 0 until a mistake is made on $(x_1,1)$. There are three cases for $x_1$. If $x_1 = 2^e y^i$ for some $e,i\in\N$ such that $i>0$ and $y \in \{5,7,11\}$, $B$ will match the function $\charfn{\{2^e,2^e 5^{c_0(e)}, 2^e y^i\}}$. Since $(x_1,1)$ is realizable iff $\varphi_e(0)\downarrow$ and $\varphi_e(e)\downarrow$, $B$'s hypothesis is computable and can be made to err at most once before the target function is determined. If $x_1 = 2^e 3^i$ for some $e,i\in\N$ such that $i>0$, $B$ will match the target function $\charfn{\{2^e, 2^e 3^i\}}$ and make no further mistakes. Finally, if $x_1 = 2^e$ for some $e\in\N$, $B$ matches $\charfn{\{2^e, 2^e 5^{c_0(e)}\}}$, which is computable since $\varphi_e(0)\downarrow$. $B$ can only be made to err on $(2^e 3^{c_0(e)},1)$, $(2^e 5^{c_0(e)},0)$, $(2^e 7^{c_e(e)},1)$, or $(2^e 11^{c_e(e)},1)$ (the last two only if $\varphi_e(e)\downarrow$), after which it will match the target function and not err again.
\end{proof}

\section{Littlestone dimension of $\HH_{init}$}
\label{appendix: Ldim of H_init}

In this section, we show that the class $\HH_{init}$ presented by \citet[Theorem 10]{sterkenburg2022characterizations} has infinite Littlestone dimension.

\begin{proposition}
    Define $\HH_{init}=\{h_s\}_{s\in\N}$, where, for each $s,x\in\N$, 
    $$
    h_s(x) = \begin{cases}
        1 &\text{if } \varphi_{x,s}(x)\downarrow \\
        0 &\text{otherwise,}
    \end{cases}
    $$
    and $\varphi_{i,s}(x)\downarrow$ denotes that $\varphi_i$ halts on input $x$ within $s$ computation steps.
    Then, $\Ldim(\HH_{init}) = \infty$.
\end{proposition}
\begin{proof}
	We say that a hypothesis class $\HH \subseteq \Yb^\X$ contains $k$ thresholds if there are $x_1,\ldots,x_k\in\X$ and $h_1,\ldots,h_k\in\HH$ such that for all $i,j \in [k]$, $h_i(x_j) = \one_{[i \geq j]}$. It is not difficult to show that if $\HH$ contains $2^n$ thresholds, then $\Ldim(\HH) \geq n$ \citep[see][Appendix A]{alon2020private}. We will show that $\Ldim(\HH_{init}) = \infty$ by showing that for each $k\in\N$, $\HH$ contains $k$ thresholds.
	
	Define $H = \{z\in\N: \varphi_z(z)\downarrow\}$ and for any $z \in H$, let $s_z = \arg\min_{s\in\N}\varphi_{z,s}(z)\downarrow$. That is, $s_z$ is the earliest time step at which $\varphi_z(z)\downarrow$. First, we will show that for each $z_1 \in H$, there exists $z_2 \in H$ such that $s_{z_2} > s_{z_1}$. That is, $\varphi_{z_2}(z_2)$ converges strictly after $\varphi_{z_1}(z_1)$. Assume by way of contradiction that there exists some $z_1 \in H$ such that for all $z_2 \in H$, $s_{z_2} \leq s_{z_1}$. Then, $H = \{z\in\N: \varphi_{z,s_{z_1}}(z)\downarrow\}$ and $\overline{H} = \{z\in\N: \varphi_{z,s_{z_1}}(z)\uparrow\}$. However, this would imply that $\overline{H}$ is recursively enumerable, which contradicts the undecidability of $H$.

	Therefore, for any $k\in\N$, there exist $x_1,\ldots,x_k \in H$ such that $s_{x_1} < \ldots < s_{x_k}$. Note that $h_{s_{x_1}},\ldots,h_{s_{x_k}}$ form $k$ thresholds over these instances, since for each $i,j\in[k]$, $h_{s_{x_i}}(x_j) = \one_{[\varphi_{x_j,s_{x_i}}(x_j)\downarrow]} = \one_{[s_{x_i} ~ \geq ~ s_{x_j}]} = \one_{[i ~ \geq ~ j]}$.
\end{proof}

\end{document}